\documentclass{article}
\usepackage{ijcai24}

\usepackage{times}
\usepackage{soul}
\usepackage{url}
\usepackage[hidelinks]{hyperref}
\usepackage[utf8]{inputenc}
\usepackage[small]{caption}
\usepackage{graphicx}
\usepackage{amsmath}
\usepackage{amsthm}
\usepackage{booktabs}
\usepackage{algorithm}
\usepackage{algorithmic}
\usepackage[switch]{lineno}
\usepackage{amsfonts}
\usepackage{multirow}
\usepackage{subfigure}
\usepackage{bbding}
\usepackage{xcolor}
\usepackage{dsfont}
\usepackage{bm}
\usepackage{colortbl}

\newtheorem{theorem}{Theorem}

\title{From Optimization to Generalization: Fair Federated Learning against \\Quality Shift via Inter-Client Sharpness Matching}

\author{
	Nannan Wu
	\and
	Zhuo Kuang\and
	Zengqiang Yan\thanks{Corresponding author}\And
	Li Yu\\
	\affiliations
	School of Electronic Information and Communications, Huazhong University of Science and Technology\\
	\emails
	\{wnn2000, kuangzhuo, z\_yan, hustlyu\}@hust.edu.cn
}

\begin{document}
	
	\maketitle
	
	\begin{abstract}
		Due to escalating privacy concerns, federated learning has been recognized as a vital approach for training deep neural networks with decentralized medical data. In practice, it is challenging to ensure consistent imaging quality across various institutions, often attributed to equipment malfunctions affecting a minority of clients. This imbalance in image quality can cause the federated model to develop an inherent bias towards higher-quality images, thus posing a severe fairness issue. In this study, we pioneer the identification and formulation of this new fairness challenge within the context of the imaging quality shift. Traditional methods for promoting fairness in federated learning predominantly focus on balancing empirical risks across diverse client distributions. This strategy primarily facilitates fair optimization across different training data distributions, yet neglects the crucial aspect of generalization. To address this, we introduce a solution termed Federated learning with Inter-client Sharpness Matching (FedISM). FedISM enhances both local training and global aggregation by incorporating sharpness-awareness, aiming to harmonize the sharpness levels across clients for fair generalization. Our empirical evaluations, conducted using the widely-used ICH and ISIC 2019 datasets, establish FedISM's superiority over current state-of-the-art federated learning methods in promoting fairness. Code is available at \textcolor{blue}{https://github.com/wnn2000/FFL4MIA}.
	\end{abstract}
	
	\section{Introduction}
	In light of escalating concerns regarding data privacy, federated learning (FL) \cite{mcmahan2017communication} has emerged as a promising approach for training deep neural networks in the realm of medical image analysis \cite{dou2021federated}. A significant challenge within FL is the inherent data heterogeneity \cite{ye2023heterogeneous,FCCL_CVPR22,FCCLPlus_TPAMI23} observed across various medical institutions, primarily attributed to their independent data collection processes. This heterogeneity has been examined from several perspectives in existing research, \textit{e.g.}, domain shift \cite{li2021fedbn,liu2021feddg,jiang2023iop}, label skew \cite{zhang2022federated,DBLP:conf/miccai/WuYYCY23}, and label quality variation \cite{DBLP:conf/ijcai/Wu0JCY23,DBLP:journals/mia/ChenLXY23,wu2023feda3i}. Nevertheless, the prevalent issue of quality heterogeneity \cite{fang2023robust} in medical imaging, a factor that could potentially raise new challenges for FL, remains under-explored.
	
	\begin{figure}[!t] 
		\centering
		\includegraphics[width=\columnwidth]{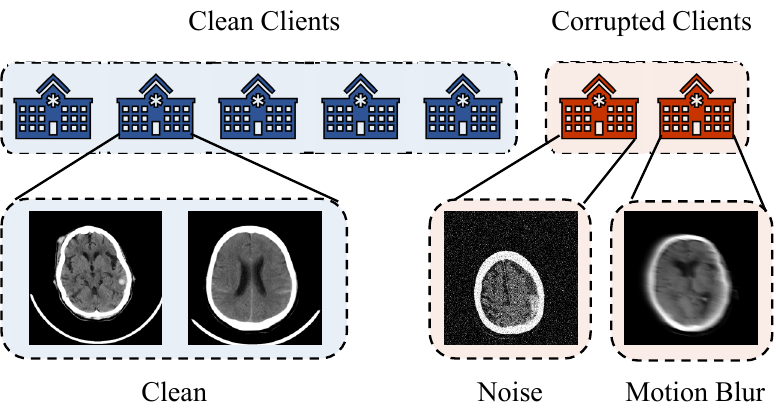}
		\caption{Imaging quality shift across clients. Most clients possess clean images, while others have corrupted images (\textit{e.g.}, exhibiting noise or blur).}
		\label{background}
	\end{figure}
	
	Despite relatively rigorous protocols in medical imaging environments, uniformity in image quality across various institutions cannot be strictly assured. As illustrated in Fig.~\ref{background}, images from some clients may exhibit noise due to equipment malfunctions or the necessity for low-dose imaging. Moreover, random movements from either patients or cameras in certain cases can result in motion blur. Typically, the proportion of low-quality (\textit{i.e.}, corrupted) images is smaller compared to high-quality (\textit{i.e.}, clean) images \cite{DBLP:conf/cvpr/HuangZX00000L23}. In such scenarios, characterized by quality shifts across clients, FL tends to exhibit a bias towards the more-prevalent clean images, thereby compromising the performance on the less-frequent corrupted images. Such a bias can be a critical concern when applying federated models to clients with corrupted images. Therefore, it is imperative to address the issue of biased performance under quality shifts.
	
	\begin{figure}[!t] 
		\centering
		\includegraphics[width=\columnwidth]{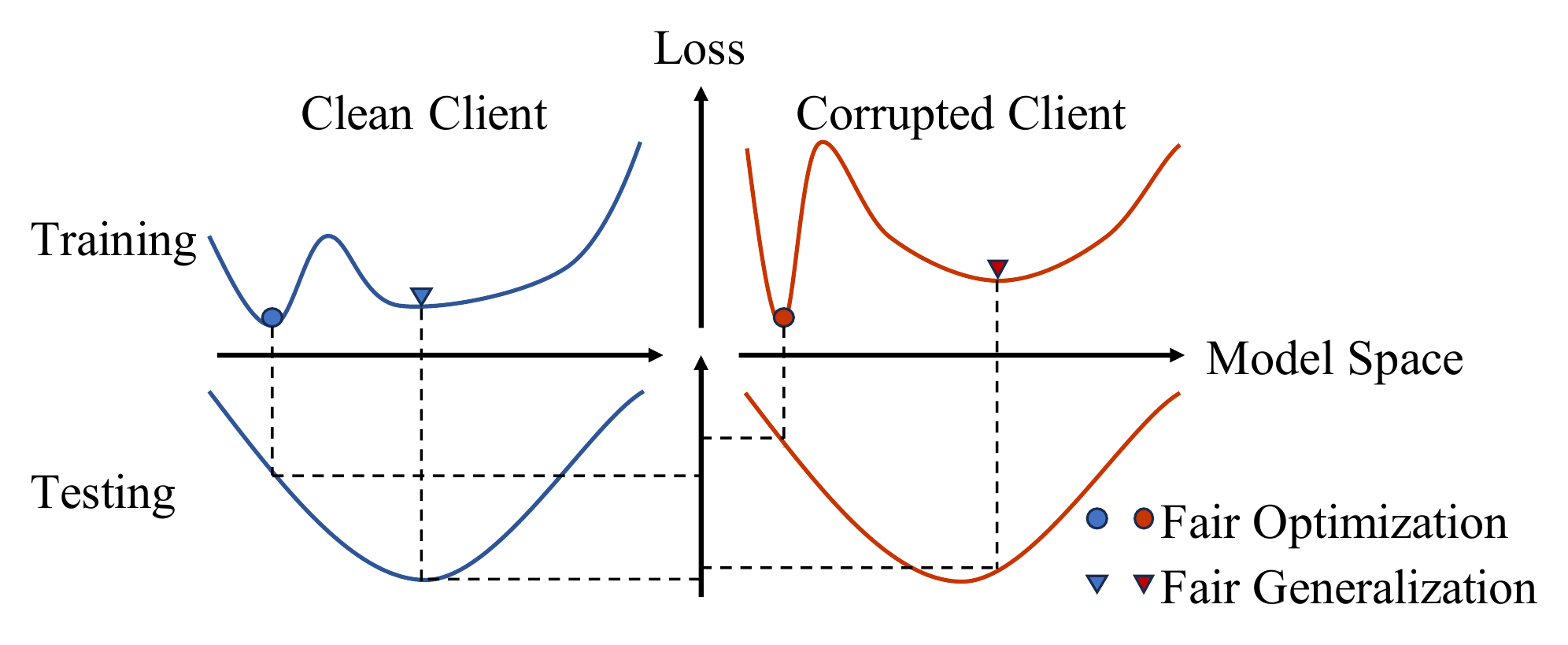}
		\caption{Motivations behind previous fair optimization and our fair generalization. Fair optimization aims for uniform and low loss values across clients, often resulting in convergence at sharp minima and poor testing performance. Comparatively, we focus on achieving uniform sharpness and converging at flat minima, thus enhancing fair generalization and testing performance.}
		\label{motivation}
	\end{figure}
	
	In this paper, we pioneer the identification and formulation of this significant challenge in FL within real-world medical contexts. In FL, especially under the mantle of privacy protection, there is an absence of prior knowledge about the imaging quality of each client, suggesting the necessity for a data-agnostic solution to this issue. We theoretically demonstrate improving performance on the poorest-quality images is equivalent to achieving client-level fairness in FL. Hence, we remodel this challenge as a problem of client-level fairness, \textit{i.e.}, \textbf{\textit{how to ensure equitable FL performance among clients with clean and corrupted image distributions?}} This investigation marks the first effort to promote fair FL across clients with image quality shifts, a departure from previous research focused on fairness under domain shifts \cite{jiang2023fair} or class distribution shifts \cite{li2019fair}. 
	
	In this work, we observe that existing methods to foster client-level performance fairness in FL typically modify the importance weights in global aggregation to harmonize certain training metrics (\textit{e.g.}, loss) across clients \cite{mohri2019agnostic,li2019fair,jiang2023fair}. However, as illustrated in Fig.~\ref{motivation}, while these methods facilitate fairness regarding optimization during training, they do not necessarily ensure fairness regarding generalization during testing. This gap stems from the issue that equalizing training losses for all clients may result in a simplistic convergence towards a sharp minimum, particularly for the minority distributions (\textit{i.e.}, corrupted image distributions), thereby impairing generalization on testing sets. To counter this, \textbf{\textit{our focus extends beyond fair optimization to fair generalization}}. Recent advancements in sharpness-aware minimization reveal an inverse relationship between the generalization capability and the sharpness of the loss surface \cite{foret2021sharpness,zhuang2022surrogate}. Inspired by this finding, we introduce a novel \textbf{Fed}erated learning framework with \textbf{I}nter-client \textbf{S}harpness \textbf{M}atching (\textbf{FedISM}), which aims to equalize the sharpness levels across clients for fair generalization. In this framework, local optimization in FL is made sharpness-aware, where each client's local update aims to minimize sharpness on its local data. Subsequently, weights for global aggregation are determined based on each client's sharpness level, with higher weights assigned to clients exhibiting greater sharpness. In this way, the global model update is more effective in minimizing sharpness in higher-level clients, leading to more uniform sharpness and fair generalization across both clean and corrupted image distributions as shown in Fig.~\ref{motivation}. It is worth mentioning that FedISM involves only simple modifications to local optimization and global aggregation compared to FedAvg \cite{mcmahan2017communication}, thus eschewing complex loss functions and additional regularization and offering ease of implementation.
	
	The contributions are summarized into three folds:
	
	\begin{itemize}
		\item New Fairness Challenge in FL - Imaging Quality Shift: We identify and articulate a new challenge in FL, specifically concentrating on achieving performance fairness across clients with diverse imaging qualities.
		
		\item Innovative Solution - FedISM: To address this challenge, our focus extends beyond previous fair optimization to fair generalization. Our proposed solution, FedISM, integrates sharpness-aware local updates with sharpness-dependent global aggregation, promoting uniform sharpness across clients and achieving more equitable generalization.
		
		\item Extensive Validation: We validate the effectiveness of FedISM through a series of experiments on two real-world medical image classification datasets, \textit{i.e.},  RSNA ICH and ISIC 2019. FedISM demonstrates superior performance, outperforming several state-of-the-art methods in promoting fairness in FL.
	\end{itemize}
	
	\section{Related Work}
	
	\subsection{Fair Federated Learning}
	Fairness \cite{huang2023federated} has become a crucial topic in FL, primarily concentrating on collaborative fairness \cite{lyu2020collaborative,xu2021gradient} and performance fairness \cite{li2019fair}. Collaborative fairness aims for the reward of each participant to be proportional to its contribution to the federation. In contrast, performance fairness advocates for unbiased performance across different devices/attributes/distributions. Existing solutions mainly achieve this goal by balancing weights of different objects to optimize fairly. For instance, q-FedAvg \cite{li2019fair} addresses this by prioritizing clients that are harder to optimize; FedCE \cite{jiang2023fair} tackles this issue by considering task-specific performance. Nevertheless, fairness respecting generalization has not been considered thoroughly till now, leading to sub-optimal outcomes.
	
	\subsection{Sharpness of Loss Surface}
	Bridging the gap between training optimization and testing generalization remains a pivotal challenge in machine learning. Recent developments in sharpness-aware minimization suggest that models tend to generalize better on flat minima than on sharp minima \cite{foret2021sharpness,zhuang2022surrogate}. Motivated by this insight, various studies have incorporated the concept of loss surface sharpness to address poor generalization. SharpDRO \cite{DBLP:conf/cvpr/HuangZX00000L23} merges sharpness with GroupDRO \cite{sagawa2019distributionally} to achieve robust generalization. ImbSAM \cite{zhou2023imbsam} focuses on enhancing the performance for tail classes in long-tailed recognition by minimizing sharpness in these classes. In addition to typical machine learning, it has recently been applied in FL \cite{qu2022generalized,caldarola2022improving,sun2023dynamic}. However, the relationship between sharpness and fairness within FL has not yet been investigated.

	\section{Methodology}
	
	\subsection{Preliminaries}
	For a typical image classification task with $C$ classes, we consider a cross-silo FL scenario with $K$ participants. Each $k$-th participant possesses a private dataset $D_k=\{(\boldsymbol x_i \in \mathcal{X}, y_i \in \mathcal{Y})\}_{i=1}^{N_k}, k \in [K]$, where $\mathcal{X}$ and $\mathcal{Y}=[C]$ represent the input image and label spaces, respectively. An image-label pair $(\boldsymbol x_i, y_i)$ is drawn from the client-specific distribution $\mathbb{P}_k(\boldsymbol x,y \mid a_k)$, with $a_k \in [A]$ indicating an attribute only influencing image quality in client $k$. This paper highlights the uneven distribution of clients across various quality attributes. Define $f(\cdot; \boldsymbol \theta): \mathcal{X} \rightarrow \Delta^{C-1}$ as a deep learning model parameterized by $\boldsymbol \theta$ and $\ell: \Delta^{C-1} \times \mathcal{Y} \rightarrow \mathbb{R_+}$ as the loss function, where $\Delta$ denotes the probability simplex. In this paper, the goal of FL is to optimize $\boldsymbol \theta$ for performance enhancement on images with the worst-performing quality:
	\begin{equation} \label{objective1}
		\boldsymbol \theta^* = \mathop{\arg\min}\limits_{\boldsymbol \theta} \{\max_{a \in [A]} \mathbb{E}_{(\boldsymbol x,y) \sim \mathbb{P}(\boldsymbol x, y \mid a)} [\ell(f(\boldsymbol x; \boldsymbol \theta), y)] \}.
	\end{equation}
	When treating each quality as a group, Eq. \ref{objective1} achieves group fairness which is similar to the aim in distributionally robust optimization \cite{sagawa2019distributionally}. However, in FL, there is no prior knowledge of group information (\textit{i.e.}, the imaging quality in each client), making group-wise design impractical. Therefore, we achieve Eq. \ref{objective1} via client fairness, as stated in the following theorem following \cite{papadaki2022minimax}:
	\begin{theorem}[Equivalence] \label{Equivalence}
		Assuming class distributions of the testing set and all clients' training sets are identical, we have:
		\begin{equation} \label{client_fairness}
			\resizebox{0.9\hsize}{!}{$
				\begin{aligned}
					&\boldsymbol \theta^*, \boldsymbol \lambda^* = \arg \mathop{\min}\limits_{\boldsymbol \theta} \mathop{\max}\limits_{\boldsymbol \lambda \in \Delta^{K-1}} \sum_{k=1}^{K} \lambda_k\mathbb{E}_{(\boldsymbol x,y) \sim \mathbb{P}_k(\boldsymbol x, y \mid a_k)} [\ell(f(\boldsymbol x; \boldsymbol \theta), y)], 
				\end{aligned}
				$}
		\end{equation}
		and
		\begin{equation} \label{group_fairness}
			\resizebox{0.9\hsize}{!}{$
				\begin{aligned}
					&\boldsymbol \theta^*, \boldsymbol \mu^* = \arg \mathop{\min}\limits_{\boldsymbol \theta} \mathop{\max}\limits_{\boldsymbol \mu \in \Delta^{A-1}} \sum_{u=1}^{A} \mu_u \mathbb{E}_{(\boldsymbol x,y) \sim \mathbb{P}(\boldsymbol x, y \mid u)} [\ell(f(\boldsymbol x; \boldsymbol \theta), y)],
				\end{aligned}
				$}
		\end{equation}
		where $\mu^*_u = \sum_{k=1}^{K} \mathds{1}_{a_k=u} \lambda^*_k$.
	\end{theorem}
	\noindent The proof is detailed in the Appendix. Notably, even when label distribution shifts occur, they can be theoretically addressed by logit adjustment \cite{menon2020long,zhang2022federated}, thereby not considered to be addressed in this paper. Theorem \ref{Equivalence} demonstrates that achieving client fairness (Eq. \ref{client_fairness}) inherently ensures group fairness (Eq. \ref{group_fairness}).  Therefore, this paper subsequently concentrates on enhancing fairness across clients with various imaging qualities.
	
	\subsection{Previous Solution: Fair Optimization}
	Client-level unfairness in FL arises from a tendency to overlook certain clients, particularly those with limited data or outlier distributions. These clients are often overlooked because optimizing them does not significantly contribute to the overall optimization objective. To address this challenge, various approaches have been proposed. AFL \cite{mohri2019agnostic} concentrates efforts on the worst-performing client; q-FedAvg \cite{li2019fair} and FairFed \cite{ezzeldin2023fairfed} give greater weights to those clients with higher training losses; and FedCE \cite{jiang2023fair} focuses on clients with lower task-specific metrics. These methods collectively aim to balance optimization across the empirical distribution of each client, striving for more uniform risks among clients, as represented by the following optimization problem:
	\begin{equation} \label{fair_optimization}
		\mathop{\min}\limits_{\boldsymbol \theta} \mathop{\max}\limits_{\boldsymbol \lambda \in \Delta^{K-1}} \sum_{k=1}^{K} \frac{\lambda_k}{N_k} \sum_{(\boldsymbol x,y) \in D_k} \ell(f(\boldsymbol x; \boldsymbol \theta), y).
	\end{equation}
	Essentially, these solutions seek to achieve a balance in optimization by prioritizing clients that are performing worse. However, this strategy may result in clients with poorer performance rapidly converging to sharp minima, primarily because the rate of loss minimization is most pronounced when moving towards them. Due to the discrepancy between empirical and expected risks, this strategy to optimize fairly does not necessarily lead to client fairness in a strict sense, as denoted by Eq. \ref{client_fairness}. A conceptual illustration of this issue is presented in Fig. \ref{motivation}. Thus, it becomes crucial to extend the exploration of client fairness to include considerations of generalization, beyond mere optimization.
	
	\subsection{Measurement of Generalization: Sharpness}
	Since it is sub-optimal to measure generalization capacity by a single value of empirical risk (\textit{i.e.}, training loss), promoting fairness by merely seeking uniformity in such a metric is inappropriate. To tackle this problem, we should identify a new indicator more closely correlated with generalization ability. 
	
	A key limitation of using single training loss as a metric is its insensitivity to the geometric properties of the loss landscape, treating sharp and flat minima indiscriminately. To overcome this, we propose focusing on a range of loss values, specifically the sharpness of the loss surface \cite{foret2021sharpness,zhuang2022surrogate}. It is defined as the largest loss change in the vicinity of the initial model parameters:
	\begin{equation} \label{sharpness}
		\mathcal{S} := \max_{\Vert \boldsymbol \epsilon \Vert_2 \leq \rho}\{\ell(f(\boldsymbol x; \boldsymbol \theta + \boldsymbol \epsilon), y) -  \ell(f(\boldsymbol x; \boldsymbol \theta), y)\},
	\end{equation}
	where $\rho$ is a positive step parameter controlling the search radius. Calculating sharpness as per Eq. \ref{sharpness} poses a challenge due to the continuous and infinite nature of the perturbation. To simplify, the difference can be approximated linearly through the Taylor series when $\rho$ is sufficiently small:
	\begin{equation} \label{approximation}
		\ell(f(\boldsymbol x; \boldsymbol \theta + \boldsymbol \epsilon), y) -  \ell(f(\boldsymbol x; \boldsymbol \theta), y) \approx \boldsymbol \epsilon^\top \nabla \ell(f(\boldsymbol x; \boldsymbol \theta), y).
	\end{equation}
	This approximation enables us to identify the optimal perturbation by maximizing the right-hand side of the equation:
	\begin{equation} \label{optimal_perturbation}
		\boldsymbol \epsilon^* = \mathop{\arg\max}\limits_{\Vert \boldsymbol \epsilon \Vert_2 \leq \rho} \boldsymbol \epsilon^\top \nabla \ell(f(\boldsymbol x; \boldsymbol \theta), y) = \rho \frac{\nabla \ell(f(\boldsymbol x; \boldsymbol \theta), y)}{\Vert\nabla \ell(f(\boldsymbol x; \boldsymbol \theta), y) \Vert_2}.
	\end{equation}
	Therefore, sharpness can be computed more feasibly as:
	\begin{equation} \label{sharpness1}
		\mathcal{S} \approx  \ell(f(\boldsymbol x; \boldsymbol \theta+ \boldsymbol \epsilon^*), y) -  \ell(f(\boldsymbol x; \boldsymbol \theta), y).
	\end{equation}
	
	Unlike a single training loss, sharpness reflects the rate of change in training loss across the loss surface. Prior research has shown that this rate of change is closely related to generalization capacity. To be specific, models generally perform better on a flat minimum (\textit{i.e.}, with smaller sharpness) than on a sharp minimum (\textit{i.e.}, with larger sharpness) \cite{foret2021sharpness,zhuang2022surrogate}. Therefore, we select this indicator to measure generalization capacity.

	\subsection{Sharpness Matching for Fair Generalization}
	
	\begin{figure}[!t] 
		\centering
		\includegraphics[width=1.0\columnwidth]{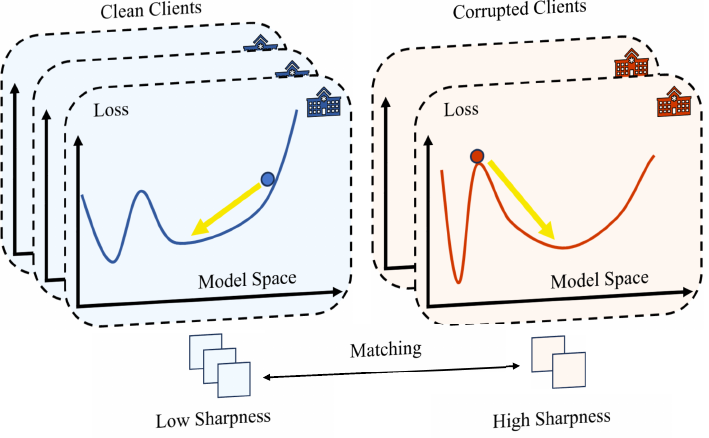}
		\caption{Illustration of FedISM. Our insight is to ensure the uniformity of sharpness across clients, leading to fair generalization.}
		\label{framework}
	\end{figure}
	
	As the sharpness of the loss surface is indicative of a model's generalization ability, we propose an FL method with inter-client sharpness matching (FedISM) to establish a uniform sharpness distribution across clients, thereby achieving fairer generalization. The overview of FedISM is depicted in Fig. \ref{framework}, and details are as follows.
	
	Our proposed method is structured around the following objective:
	\begin{equation} \label{fair_generalization}
		\mathop{\min}\limits_{\boldsymbol \theta} \mathop{\max}\limits_{\boldsymbol \lambda \in \Delta^{K-1}} \sum_{k=1}^{K} \frac{\lambda_k}{N_k} \sum_{(\boldsymbol x,y) \in D_k} \mathcal{S}(f(\boldsymbol x; \boldsymbol \theta), y),
	\end{equation}
	where $\mathcal{S}(f(\boldsymbol x; \boldsymbol \theta), y)$ represents the sharpness for a given sample and model, as defined in Eq. \ref{sharpness1}. Unlike previous solutions which focus on a uniformly-low loss (Eq. \ref{fair_optimization}), our key objective is to achieve a uniformly-low sharpness across clients, thereby avoiding convergence to sharp minima while aligning empirical risks. Considering that lower sharpness typically reflects better generalization, this strategy emphasizes fairness regarding generalization more.
	
	\begin{algorithm}[!t]
		\caption{Pseudocode of \colorbox[rgb]{1.0, 0.90, 0.87}{$\texttt{FedAvg}$} and our \colorbox[rgb]{0.87, 0.90, 1}{$\texttt{FedISM}$}}
		\label{alg:FedISM}
		\textbf{Input}: number of clients $K$, local datasets $\{D_1, \dots, D_K\}$, local dataset size, total communication rounds $T$, learning rate of local training $\eta$.\\
		\textbf{Output}: final global model $\boldsymbol \theta_{T+1}$
		\begin{algorithmic}[1] 
			\STATE Initialize the global model $\boldsymbol \theta_1$
			\FOR{$t = 1, 2, \dots, T$}
			\FOR{Client $k = 1, 2, \dots, K$ in parallel}
			\STATE $\boldsymbol \theta_{(t,k)} \leftarrow \boldsymbol \theta_t$ \hfill \textcolor{blue}{$\triangleright$ download the global model}
			\FOR{$(\boldsymbol x_i, y_i) \in D_k$}
			\STATE \colorbox[rgb]{1.0, 0.90, 0.87}{Update $\boldsymbol \theta_{(t,k)}$ with $(\boldsymbol x_i, y_i)$ by Eq. \ref{GD}}
			\STATE \colorbox[rgb]{0.87, 0.90, 1}{Update $\boldsymbol \theta_{(t,k)}$ with $(\boldsymbol x_i, y_i)$ by Eq. \ref{SAM}}
			\ENDFOR
			\ENDFOR
			\STATE \colorbox[rgb]{1.0, 0.90, 0.87}{$\boldsymbol \theta_{t+1} \leftarrow$ Aggregate $\{\boldsymbol \theta_{(t,k)}\}_{k=1}^K$ with $\boldsymbol w_{\texttt{Avg}}$ (Eq. \ref{FedAvg_weights})}
			\STATE \colorbox[rgb]{0.87, 0.90, 1}{$\boldsymbol \theta_{t+1} \leftarrow$ Aggregate $\{\boldsymbol \theta_{(t,k)}\}_{k=1}^K$ with $\boldsymbol w$ (Eq. \ref{FedISM_weights1})}
			\ENDFOR
			\RETURN $\boldsymbol \theta_{T+1}$
		\end{algorithmic}
	\end{algorithm}
	
	In conventional local training in FL, \textit{e.g.}, FedAvg \cite{mcmahan2017communication}, the objective is to minimize the loss on local data via gradient descent:
	\begin{equation} \label{GD}
		\boldsymbol \theta \leftarrow \boldsymbol \theta - \eta \nabla \ell(f(\boldsymbol x; \boldsymbol \theta), y),
	\end{equation}
	with $\eta$ as the learning rate. However, it tends to find an isolated point of low loss, without aiming for a flat minimum as desired in Eq. \ref{fair_generalization}. To address this, local training should incorporate sharpness-awareness. Inspired by sharpness-aware minimization \cite{foret2021sharpness}, the update rule is adapted to:
	\begin{equation} \label{SAM}
		\boldsymbol \theta \leftarrow \boldsymbol \theta - \eta \nabla \ell(f(\boldsymbol x; \boldsymbol \theta + \boldsymbol \epsilon^*), y),
	\end{equation}
	where $\boldsymbol \epsilon^*$ represents the optimal perturbation that maximizes the change in the training loss as Eq. \ref{optimal_perturbation}. This modification directs the optimization towards a point where the surrounding area exhibits the lowest loss, effectively minimizing the sharpness of the loss surface.
	
	After each round of local training, gradients from all clients are transmitted to the server for global aggregation. Notably, under sharpness-aware minimization (Eq. \ref{SAM}), each client's gradient is representative of the direction to minimize sharpness with respect to local data. The crux of achieving uniform sharpness across clients lies in how these gradients are fairly aggregated. Traditional FedAvg \cite{mcmahan2017communication} assigns aggregation weights based on the quantity of data each client contributes, as defined by:
	\begin{equation} \label{FedAvg_weights}
		\boldsymbol w_{\texttt{Avg}} = \frac{1}{ \sum_{k=1}^{K} N_k}[N_1, N_2, \cdots, N_K]^\top.
	\end{equation}
	However, such a weighting policy may not facilitate uniform sharpness. Given that clients with corrupted images typically have less data, the global update tends to be dominated by clients with clean images. Consequently, this might only reduce the sharpness for clean clients, leading to a disparity in sharpness similar to the loss disparity addressed by fair optimization \cite{mohri2019agnostic,li2019fair,ezzeldin2023fairfed,jiang2023fair}. To solve this, we propose sharpness-aware aggregation:
	\begin{equation} \label{FedISM_weights}
		\widetilde{\boldsymbol w}_{t} = \frac{1}{ \sum_{k=1}^{K} \mathcal{S}_{k,t}^q}[\mathcal{S}_{1,t}^q, \mathcal{S}_{2,t}^q, \cdots, \mathcal{S}_{K,t}^q]^\top,
	\end{equation}
	where $\mathcal{S}_{k,t}$ denotes the sharpness calculated on the entire local dataset $D_k$ at any round $t$ by Eq. \ref{sharpness1}, and $q$ is a predetermined positive parameter. This strategy emphasizes clients with higher sharpness levels during aggregation. As $q$ increases, the aggregation process focuses more on these high-sharpness clients. In extreme cases, aggregation will exclusively focus on the client exhibiting the highest sharpness given $q \rightarrow + \infty$. To conclude, this method ensures that the sharpness levels of clients with initially higher sharpness are primarily reduced, thus promoting a uniform sharpness distribution across all clients. In this way, it accomplishes the objective defined in Eq. \ref{fair_generalization}. To maintain stability in federated training, a moving average is further employed for rounds $t>1$, formulated as:
	\begin{equation} \label{FedISM_weights1}
		{\boldsymbol w}_{t} = \beta \widetilde{\boldsymbol w}_{t} + (1-\beta) {\boldsymbol w}_{t-1}.
	\end{equation}
	In particular, we set ${\boldsymbol w}_{1} = \widetilde{\boldsymbol w}_{1}$ for the first round.
	
	It is important to note that FedISM only requires clients to share their sharpness values instead of any  information regarding data distributions (\textit{e.g.}, imaging quality of each client), which is privacy-preserving.
	
	To facilitate a clearer understanding, the processes of both FedISM and the conventional FedAvg \cite{mcmahan2017communication} are summarized in Algorithm \ref{alg:FedISM}. Compared to FedAvg, FedISM modifies only the local optimizer and global aggregation weights, without the need to introduce complex loss functions or additional regularization techniques. Such a streamlined approach enhances the ease of implementation.

	\section{Experiments}
	We conduct a series of experiments to assess the effectiveness of FedISM. More results and discussion can be found in Appendix.
	
	\subsection{Experimental Setup}
	
	\subsubsection{Datasets}
	Two medical datasets are used for evaluation, in line with prior FL research \cite{jiang2022dynamic,DBLP:conf/ijcai/Wu0JCY23}:
	\begin{itemize}
		\item RSNA ICH \cite{flanders2020construction}: The task is to classify each CT slice into five intracranial hemorrhage (ICH) subtypes. Following \cite{jiang2022dynamic}, we randomly select 25000 images for experiments.
		\item ISIC 2019 \cite{tschandl2018ham10000,codella2018skin,combalia2019bcn20000}: This dataset contains 25331 images for developing models to classify eight skin diseases.
	\end{itemize}
	Both datasets are split into training and test sets in an 8:2 ratio and resized to $224 \times 224$ pixels following the standard operation \cite{jiang2022dynamic}. For data partitioning, training sets are partitioned into 20 clients using a Dirichlet distribution (\textit{i.e.}, $Dir(1.0)$), simulating the prevalent label distribution shifts though this paper does not focus mainly on it. Imaging quality shifts are created via Gaussian noise added on a subset of clients following \cite{hendrycks2019benchmarking}.
	
	\subsubsection{Model}
	For standard evaluation, pretrained ResNet-18 \cite{he2016deep} is used as the base model for all experiments.
	
	\subsubsection{Implement Details}
	To mitigate label distribution shifts between local training and testing sets, we incorporate logit adjustment \cite{menon2020long} in the training of local models. We train batches of 32 images using the Adam optimizer, with a constant learning rate of 0.0003, beta values of (0.9, 0.999), and a weight decay of 0.0005. For FL specifics, we set a maximum of 300 communication rounds and the local epoch as 1. These parameters are consistent across all experiments to facilitate fair comparison. In FedISM, we implement GSAM \cite{zhuang2022surrogate} for sharpness-aware minimization, with default settings of $q$ = 2.0 and $\beta$ = 0.5.
	
	\subsubsection{Evaluation Strategy}
	The federated model's performance is evaluated on both an unaltered clean testing set and a generated corrupted testing set (with the identically distributed Gaussian noise in corrupted clients). We use class-balanced accuracy (ACC) and the area under the receiver operating characteristic curve (AUC) as our evaluation metrics. To ensure the robustness of results, three independent experiments are conducted, and performance is averaged over the last five communication rounds, in line with \cite{huang2023rethinking}.
	
	\subsection{Comparison to State-of-the-Arts} \label{sec:sota}
	To validate the superiority, we compare FedISM with several leading methods, including the basic FL approach FedAvg \cite{mcmahan2017communication}, and five advanced fair FL methods: Agnostic-FL \cite{mohri2019agnostic}, q-FedAvg \cite{li2019fair}, FairFed \cite{ezzeldin2023fairfed}, FedCE \cite{jiang2023fair}, and FedGA \cite{zhang2023federated}. Comprehensive details on these methods and implementation strategies are available in the Appendix. In our experiments, 4 out of 20 clients are equipped with corrupted images, constituting a corrupted client ratio of 20\%.
	
	Tab. \ref{tab:SOTA} summarizes the mean and standard deviation of ACC and AUC for both clean and corrupted images, as well as their average. Although most fair FL methods improve the performance on corrupted images, excluding the less stable Agnostic-FL, they focus primarily on fair optimization rather than generalization, which can be suboptimal. Comparatively, our FedISM, with its emphasis on sharpness minimization, achieves better generalization on corrupted images. For instance, on the ICH dataset, FedISM outperforms FedAvg and the second-best method (FedGA) by 10.85\% and 3.78\%, respectively. Notably, previous methods often boost the performance on corrupted images at the expense of performance degradation on clean images in this setting. For instance, FedGA's ACC on clean ICH images falls by 4.84\% compared to FedAvg. This trade-off can be problematic in medical scenarios, potentially discouraging high-quality institutions from participation in FL. In contrast, FedISM not only enhances the performance on corrupted images but also maintains top results on clean images. This balanced improvement is crucial in medical scenarios, ensuring accurate diagnostics and encouraging broader participation in FL.
	
	\begin{table*}[!t]
		\centering
		\resizebox{1.0\textwidth}{!}{
			\begin{tabular}{c|c|cccccccccccc}
				\toprule
				\hline
				\multirow{4}{*}{Category} & \multirow{4}{*}{Method} & \multicolumn{12}{c}{Dataset (Corrupted Clients Ratio: 20\%)} \\ \cline{3-14} 
				& & \multicolumn{6}{c|}{ICH} & \multicolumn{6}{c}{ISIC 2019} \\ \cline{3-14} 
				& & \multicolumn{2}{c}{Clean} & \multicolumn{2}{c}{Corrupted} & \multicolumn{2}{c|}{Average} & \multicolumn{2}{c}{Clean} & \multicolumn{2}{c}{Corrupted} & \multicolumn{2}{c}{Average} \\ \cline{3-14} 
				& & ACC $\uparrow$ & AUC $\uparrow$ & ACC $\uparrow$ & AUC $\uparrow$ & ACC $\uparrow$ & \multicolumn{1}{c|}{AUC $\uparrow$} & ACC $\uparrow$ & AUC $\uparrow$ & ACC $\uparrow$ & AUC $\uparrow$ & ACC $\uparrow$ & AUC $\uparrow$ \\ \hline
				\multirow{2}{*}{Naive FL} & FedAvg & $76.77$ & $94.58$ & $53.66$ & $84.37$ & $65.21$ & \multicolumn{1}{c|}{$89.48$} & $64.43$ & $91.91$ & $38.26$ & $78.03$ & $51.35$ & $84.97$ \\
				& (AISTATS'17) & $(0.68)$ & $(0.20)$ & $(1.28)$ & $(0.73)$ & $(0.75)$ & \multicolumn{1}{c|}{$(0.34)$} & $(1.01)$ & $(0.40)$ & $(1.57)$ & $(1.32)$ & $(0.70)$ & $(0.64)$ \\ \hline
				\multirow{10}{*}{Fair FL} & Agnostic-FL & $55.13$ & $83.20$ & $46.69$ & $78.13$ & $50.91$ & \multicolumn{1}{c|}{$80.67$} & $39.66$ & $78.64$ & $33.55$ & $75.62$ & $36.60$ & $77.13$ \\
				& (ICML'19) & $(7.28)$ & $(4.86)$ & $(7.54)$ & $(4.79)$ & $(2.56)$ & \multicolumn{1}{c|}{$(1.65)$} & $(12.94)$ & $(9.08)$ & $(8.88)$ & $(6.48)$ & $(3.30)$ & $(2.52)$ \\ \cline{2-14} 
				& q-FedAvg & $75.94$ & $94.36$ & $59.46$ & $86.81$ & $67.70$ & \multicolumn{1}{c|}{$90.58$} & $65.20$ & $91.59$ & $44.54$ & $82.88$ & $54.87$ & $87.24$ \\
				& (ICLR'20) & $(1.05)$ & $(0.28)$ & $(1.90)$ & $(0.95)$ & $(0.69)$ & \multicolumn{1}{c|}{$(0.40)$} & $(1.26)$ & $(0.67)$ & $(0.80)$ & $(0.61)$ & $(0.58)$ & $(0.47)$ \\ \cline{2-14} 
				& FairFed & $74.13$ & $93.55$ & $60.27$ & $86.74$ & $67.20$ & \multicolumn{1}{c|}{$90.15$} & $60.36$ & $90.29$ & $49.04$ & $84.86$ & $54.70$ & $87.58$ \\
				& (AAAI'23) & $(1.15)$ & $(0.30)$ & $(1.08)$ & $(0.63)$ & $(0.79)$ & \multicolumn{1}{c|}{$(0.38)$} & $(1.59)$ & $(0.54)$ & $(1.80)$ & $(0.64)$ & $(1.39)$ & $(0.52)$ \\ \cline{2-14}
				& FedCE & $75.82$ & $94.31$ & $58.77$ & $86.93$ & $67.29$ & \multicolumn{1}{c|}{$90.62$} & $62.21$ & $90.37$ & $45.25$ & $83.37$ & $53.73$ & $86.87$ \\
				& (CVPR'23) & $(0.45)$ & $(0.10)$ & $(1.92)$ & $(0.70)$ & $(0.94)$ & \multicolumn{1}{c|}{$(0.37)$} & $(1.27)$ & $(0.56)$ & $(2.65)$ & $(1.72)$ & $(1.11)$ & $(0.91)$ \\ \cline{2-14}
				& FedGA & $71.93$ & $92.88$ & $60.73$ & $87.45$ & $66.33$ & \multicolumn{1}{c|}{$90.16$} & $59.56$ & $89.53$ & $48.16$ & $84.64$ & $53.86$ & $87.08$ \\
				& (CVPR'23) & $(1.43)$ & $(0.36)$ & $(1.24)$ & $(0.47)$ & $(0.99)$ & \multicolumn{1}{c|}{$(0.33)$} & $(1.55)$ & $(0.55)$ & $(1.38)$ & $(0.47)$ & $(1.03)$ & $(0.47)$ \\ \hline
				\multirow{2}{*}{Ours} & \multirow{2}{*}{FedISM} & $\bm{77.52}$ & $\bm{95.02}$ & $\bm{64.51}$ & $\bm{89.56}$ & $\bm{71.01}$ & \multicolumn{1}{c|}{$\bm{92.29}$} & $\bm{66.94}$ & $\bm{93.21}$ & $\bm{51.62}$ & $\bm{86.89}$ & $\bm{59.28}$ & $\bm{90.05}$ \\
				& & $\bm{(0.58)}$ & $\bm{(0.16)}$ & $\bm{(0.64)}$ & $\bm{(0.22)}$ & $\bm{(0.36)}$ & \multicolumn{1}{c|}{$\bm{(0.13)}$} & $\bm{(0.84)}$ & $\bm{(0.32)}$ & $\bm{(1.39)}$ & $\bm{(0.49)}$ & $\bm{(0.52)}$ & $\bm{(0.18)}$ \\ \hline
				\bottomrule
			\end{tabular}
		}
		\caption{Comparison of the mean (\%) and standard deviation (\%) for ACC and AUC against state-of-the-art methods. Values without parentheses represent the mean, while those within parentheses indicate the standard deviation.}
		\label{tab:SOTA}
	\end{table*}

	\subsection{Ablation Study}
	FedISM is comprised of two key components: Sharpness-Aware Local Training (SALT, Eq. \ref{SAM}) and Sharpness-Aware Global Aggregation (SAGA, Eqs. \ref{FedISM_weights} and \ref{FedISM_weights1}). To evaluate the effectiveness of each component, we conduct ablation studies by integrating them individually with FedAvg, with performance summaries provided in Table \ref{tab:Ablation}. SALT, by directing the model towards flat minima, enhances FedAvg's performance on both clean and corrupted images. Nonetheless, this approach is not fully optimized for fairness, as it fails to address the disparity in sharpness between clean and corrupted image distributions. This is where SAGA comes into play; it helps the model prioritize distributions with poorer generalization capacity, typically those of corrupted images. As a result, the combination of both SALT and SAGA yields the best performance on corrupted images, demonstrating the comprehensive effectiveness of FedISM's designs.
	
	\begin{table*}[!t]
		\centering
		\resizebox{1.0\textwidth}{!}{
			\begin{tabular}{c|cccccccccccc}
				\toprule
				\hline
				\multirow{3}{*}{Method}            & \multicolumn{12}{c}{ICH (Corrupted Clients Ratio: 20\%)}                                                                                                                                                                                             \\ \cline{2-13} 
				& \multicolumn{4}{c|}{Clean}                                                             & \multicolumn{4}{c|}{Crorrupted}                                                         & \multicolumn{4}{c}{Average}                                       \\ \cline{2-13} 
				& ACC $\uparrow$    & $\Delta$ACC $\uparrow$              & AUC $\uparrow$    & \multicolumn{1}{c|}{$\Delta$AUC $\uparrow$}              & ACC $\uparrow$    & $\Delta$ACC $\uparrow$               & AUC $\uparrow$    & \multicolumn{1}{c|}{$\Delta$AUC $\uparrow$}              & ACC $\uparrow$    & $\Delta$ACC $\uparrow$              & AUC $\uparrow$    & $\Delta$AUC $\uparrow$              \\ \hline
				FedAvg                             & $76.77$  & \multirow{2}{*}{-}     & $94.58$  & \multicolumn{1}{c|}{\multirow{2}{*}{-}}     & $53.66$  & \multirow{2}{*}{-}      & $84.37$  & \multicolumn{1}{c|}{\multirow{2}{*}{-}}     & $65.21$  & \multirow{2}{*}{-}     & $89.48$  & \multirow{2}{*}{-}     \\
				(AISTATS'17)                       & $(0.68)$ &                        & $(0.20)$ & \multicolumn{1}{c|}{}                       & $(1.28)$ &                         & $(0.73)$ & \multicolumn{1}{c|}{}                       & $(0.75)$ &                        & $(0.34)$ &                        \\ \hline
				FedAvg                             & $\bm{79.21}$  & \multirow{2}{*}{$\bm{+2.44}$} & $\bm{95.66}$  & \multicolumn{1}{c|}{\multirow{2}{*}{$\bm{+1.08}$}} & $55.29$  & \multirow{2}{*}{$+1.63$}  & $86.61$  & \multicolumn{1}{c|}{\multirow{2}{*}{$+2.24$}} & $67.25$  & \multirow{2}{*}{$+2.04$} & $91.14$  & \multirow{2}{*}{$+1.66$} \\
				+SALT     & $\bm{(0.21)}$ &                        & $\bm{(0.10)}$ & \multicolumn{1}{c|}{}                       & $(2.05)$ &                         & $(0.70)$ & \multicolumn{1}{c|}{}                       & $(1.08)$ &                        & $(0.32)$ &                        \\ \hline
				FedAvg                             & $75.45$  & \multirow{2}{*}{$-1.32$} & $93.91$  & \multicolumn{1}{c|}{\multirow{2}{*}{$-0.67$}} & $61.24$  & \multirow{2}{*}{$+7.58$}  & $87.46$  & \multicolumn{1}{c|}{\multirow{2}{*}{$+3.09$}} & $68.34$  & \multirow{2}{*}{$+3.13$} & $90.68$  & \multirow{2}{*}{$+1.20$} \\
				+SAGA & $(1.20)$ &                        & $(0.49)$ & \multicolumn{1}{c|}{}                       & $(1.81)$ &                         & $(0.95)$& \multicolumn{1}{c|}{}                       & $(0.60)$ &                        & $(0.26)$ &                        \\ \hline
				FedISM                             & $77.52$  & \multirow{2}{*}{$+0.75$} & $95.02$  & \multicolumn{1}{c|}{\multirow{2}{*}{$+0.44$}} & $\bm{64.51}$  & \multirow{2}{*}{$\bm{+10.85}$} & $\bm{89.56}$  & \multicolumn{1}{c|}{\multirow{2}{*}{$\bm{+5.19}$}} & $\bm{71.01}$  & \multirow{2}{*}{$\bm{+5.80}$} & $\bm{92.29}$  & \multirow{2}{*}{$\bm{+2.81}$} \\
				(Ours)                             & $(0.58)$ &                        & $(0.16)$ & \multicolumn{1}{c|}{}                       & $\bm{(0.64)}$ &                         & $\bm{(0.22)}$ & \multicolumn{1}{c|}{}                       & $\bm{(0.36)}$ &                        & $\bm{(0.13)}$ &                        \\ \hline
				\bottomrule
			\end{tabular}
		}
		\caption{Component-wise ablation study in mean (\%) and standard deviation (\%) in ACC and AUC. Values without parentheses represent the mean, while those within parentheses indicate the standard deviation. $\Delta$ACC and $\Delta$AUC represent the difference to ACC and AUC of FedAvg, respectively.}
		\label{tab:Ablation}
	\end{table*}

	\subsection{SALT Helps Fair Optimization}
	A key motivation for developing Sharpness-Aware Local Training (SALT) is to address the issue of fair optimization methods converging to sharp minima, often resulting in suboptimal performance (see Fig. \ref{motivation}). To validate this premise and demonstrate SALT's effectiveness, we integrate SALT with existing fair optimization methods and analyze the impact on performance enhancement. The results, presented in Tab. \ref{tab:Combining}, reveal that, aside from Agnostic-FL \cite{mohri2019agnostic} which shows large standard deviations, other methods exhibit performance improvements when combined with SALT. It not only confirms that SALT is an effective and adaptable component for fair FL, but also validates our initial motivation for its design.
	
	\begin{table}[!t]
		\centering
		\resizebox{1.0\columnwidth}{!}{
		\renewcommand{\arraystretch}{0.993}
		\begin{tabular}{l|cccc}
			\toprule
			\hline
			& \multicolumn{4}{c}{ICH (Corrupted Clients Ratio:   20\%)}                                                            \\ \cline{2-5} 
			& \multicolumn{4}{c}{Corrupted}                                                                                        \\ \cline{2-5} 
			\multirow{-3}{*}{Method}                                              & ACC     & $\Delta$ACC $\uparrow$                                       & AUC    & $\Delta$AUC $\uparrow$                                       \\ \hline
			& $53.66$   &                                                 & $84.37$  &                                                 \\
			\multirow{-2}{*}{FedAvg}                                              & $(1.28)$  & \multirow{-2}{*}{-}                             & $(0.73)$ & \multirow{-2}{*}{-}                             \\ \hline
			\rowcolor[HTML]{E4E4E4} 
			\multicolumn{1}{r|}{\cellcolor[HTML]{E4E4E4}}                         & $55.29$   & \cellcolor[HTML]{E4E4E4}                        & $86.61$  & \cellcolor[HTML]{E4E4E4}                        \\
			\rowcolor[HTML]{E4E4E4} 
			\multicolumn{1}{r|}{\multirow{-2}{*}{\cellcolor[HTML]{E4E4E4}+ SALT}} & $(2.05)$  & \multirow{-2}{*}{\cellcolor[HTML]{E4E4E4}$+1.63$} & $(0.70)$ & \multirow{-2}{*}{\cellcolor[HTML]{E4E4E4}$+2.24$} \\ \hline
			& $46.69$   &                                                 & $78.13$  &                                                 \\
			\multirow{-2}{*}{Agnostic-FL}                                         & $(7.54)$  & \multirow{-2}{*}{-}                             & $(4.79)$ & \multirow{-2}{*}{-}                             \\ \hline
			\rowcolor[HTML]{E4E4E4} 
			\multicolumn{1}{r|}{\cellcolor[HTML]{E4E4E4}}                         & $45.91$   & \cellcolor[HTML]{E4E4E4}                        & $78.23$  & \cellcolor[HTML]{E4E4E4}                        \\
			\rowcolor[HTML]{E4E4E4} 
			\multicolumn{1}{r|}{\multirow{-2}{*}{\cellcolor[HTML]{E4E4E4}+ SALT}} & $(10.89)$ & \multirow{-2}{*}{\cellcolor[HTML]{E4E4E4}$-0.78$} & $(6.90)$ & \multirow{-2}{*}{\cellcolor[HTML]{E4E4E4}$+0.10$} \\ \hline
			& $59.46$   &                                                 & $86.81$  &                                                 \\
			\multirow{-2}{*}{q-FedAvg}                                            & $(1.90)$  & \multirow{-2}{*}{-}                             & $(0.95)$ & \multirow{-2}{*}{-}                             \\ \hline
			\rowcolor[HTML]{E4E4E4} 
			\multicolumn{1}{r|}{\cellcolor[HTML]{E4E4E4}}                         & $61.30$   & \cellcolor[HTML]{E4E4E4}                        & $88.66$  & \cellcolor[HTML]{E4E4E4}                        \\
			\rowcolor[HTML]{E4E4E4} 
			\multicolumn{1}{r|}{\multirow{-2}{*}{\cellcolor[HTML]{E4E4E4}+ SALT}} & $(1.57)$  & \multirow{-2}{*}{\cellcolor[HTML]{E4E4E4}$+1.84$} & $(0.51)$ & \multirow{-2}{*}{\cellcolor[HTML]{E4E4E4}$+1.85$} \\ \hline
			& $60.27$   &                                                 & $86.74$  &                                                 \\
			\multirow{-2}{*}{FairFed}                                             & $(1.08)$  & \multirow{-2}{*}{-}                             & $(0.63)$ & \multirow{-2}{*}{-}                             \\ \hline
			\rowcolor[HTML]{E4E4E4} 
			\multicolumn{1}{r|}{\cellcolor[HTML]{E4E4E4}}                         & $62.86$   & \cellcolor[HTML]{E4E4E4}                        & $89.14$  & \cellcolor[HTML]{E4E4E4}                        \\
			\rowcolor[HTML]{E4E4E4} 
			\multicolumn{1}{r|}{\multirow{-2}{*}{\cellcolor[HTML]{E4E4E4}+ SALT}} & $(1.08)$  & \multirow{-2}{*}{\cellcolor[HTML]{E4E4E4}$+2.59$} & $(0.39)$ & \multirow{-2}{*}{\cellcolor[HTML]{E4E4E4}$+2.40$} \\ \hline
			& $58.77$   &                                                 & $86.93$  &                                                 \\
			\multirow{-2}{*}{FedCE}                                               & $(1.92)$ & \multirow{-2}{*}{-}                             & $(0.70)$ & \multirow{-2}{*}{-}                             \\ \hline
			\rowcolor[HTML]{E4E4E4} 
			\multicolumn{1}{r|}{\cellcolor[HTML]{E4E4E4}}                         & $62.34$   & \cellcolor[HTML]{E4E4E4}                        & $88.81$  & \cellcolor[HTML]{E4E4E4}                        \\
			\rowcolor[HTML]{E4E4E4} 
			\multicolumn{1}{r|}{\multirow{-2}{*}{\cellcolor[HTML]{E4E4E4}+ SALT}} & $(0.94)$  & \multirow{-2}{*}{\cellcolor[HTML]{E4E4E4}$+3.57$} & $((0.41)$ & \multirow{-2}{*}{\cellcolor[HTML]{E4E4E4}$+1.88$} \\ \hline
			& $60.73$   &                                                 & $87.45$  &                                                 \\
			\multirow{-2}{*}{FedGA}                                               & $(1.24)$  & \multirow{-2}{*}{-}                             & $(0.47)$ & \multirow{-2}{*}{-}                             \\ \hline
			\rowcolor[HTML]{E4E4E4} 
			\multicolumn{1}{r|}{\cellcolor[HTML]{E4E4E4}}                         & $61.76$   & \cellcolor[HTML]{E4E4E4}                        & $88.56$  & \cellcolor[HTML]{E4E4E4}                        \\
			\rowcolor[HTML]{E4E4E4} 
			\multicolumn{1}{r|}{\multirow{-2}{*}{\cellcolor[HTML]{E4E4E4}+ SALT}} & $(0.78)$  & \multirow{-2}{*}{\cellcolor[HTML]{E4E4E4}$+1.03$} & $(0.15)$ & \multirow{-2}{*}{\cellcolor[HTML]{E4E4E4}$+1.11$} \\ \hline
			\bottomrule
		\end{tabular}
		}
		\caption{Performance enhancement in mean (\%) and standard deviation (\%) on corrupted images by combining SALT with existing FL for fair optimization.}
		\label{tab:Combining}
	\end{table}

	\subsection{Robustness to Different Imbalanced Ratios}
	
	Our experiments also demonstrate the robustness of FedISM, showcasing stable performance enhancements across varying ratios of clean and corrupted clients. Quantitative results under various ratios of clients with corrupted images through \{0.1, 0.2, 0.3\} are illustrated in Fig. \ref{robustness}. Across all ratios and on both image distributions, FedISM consistently outperforms other methods in the two evaluation metrics. Such findings further highlight FedISM's strong adaptability and effectiveness irrespective of the corruption ratio.

	\begin{figure}[!t] 
		\centering
		\includegraphics[width=1.0\columnwidth]{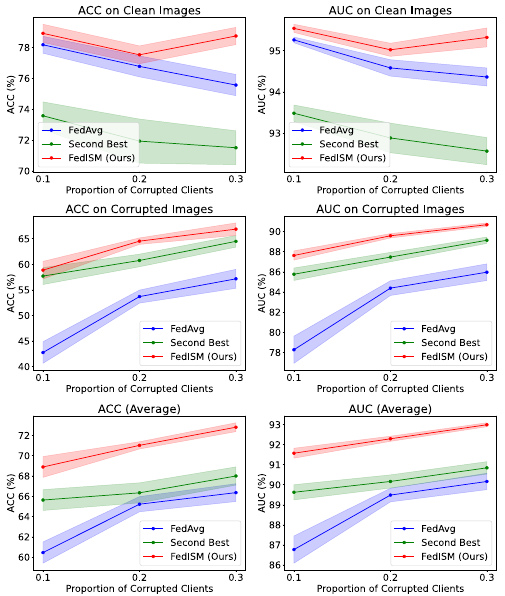}
		\caption{Evaluation across diverse ratios of clean and corrupted clients. Solid lines denote the mean values, while the transparent areas depict the standard deviations. Second best refers to the second best fair FL methods from Section \ref{sec:sota}, specifically in terms of their performance on corrupted images.}
		\label{robustness}
	\end{figure}

	\subsection{Discussion on Parameters}
	
	The parameter $q$ is closely related to how strongly FedISM concentrates on clients with higher sharpness levels. We examine this impact by varying $q$ through \{0.1, 0.5, 1.0, 2.0, 5.0, 10.0\} and assess the performance on the ICH dataset with 20\% corrupted clients, as shown in Fig. \ref{fig:q}. For comparison, we also report the performance of FedAvg and the second-best fair FL method. As $q$ increases, FedISM increasingly focuses on clients with greater sharpness (typically those with corrupted images), which results in a decrease in performance on clean images and an improvement on corrupted images. It is important to note that using very high values of $q$, such as 5 and 10, would slightly degrade the performance on corrupted images, likely due to training instability. Overall, FedISM demonstrates improved performance across a wide range of $q$ values, reducing the burden of parameter tuning in practice.
	
	\begin{figure}[!t] 
		\centering
		\includegraphics[width=1.0\columnwidth]{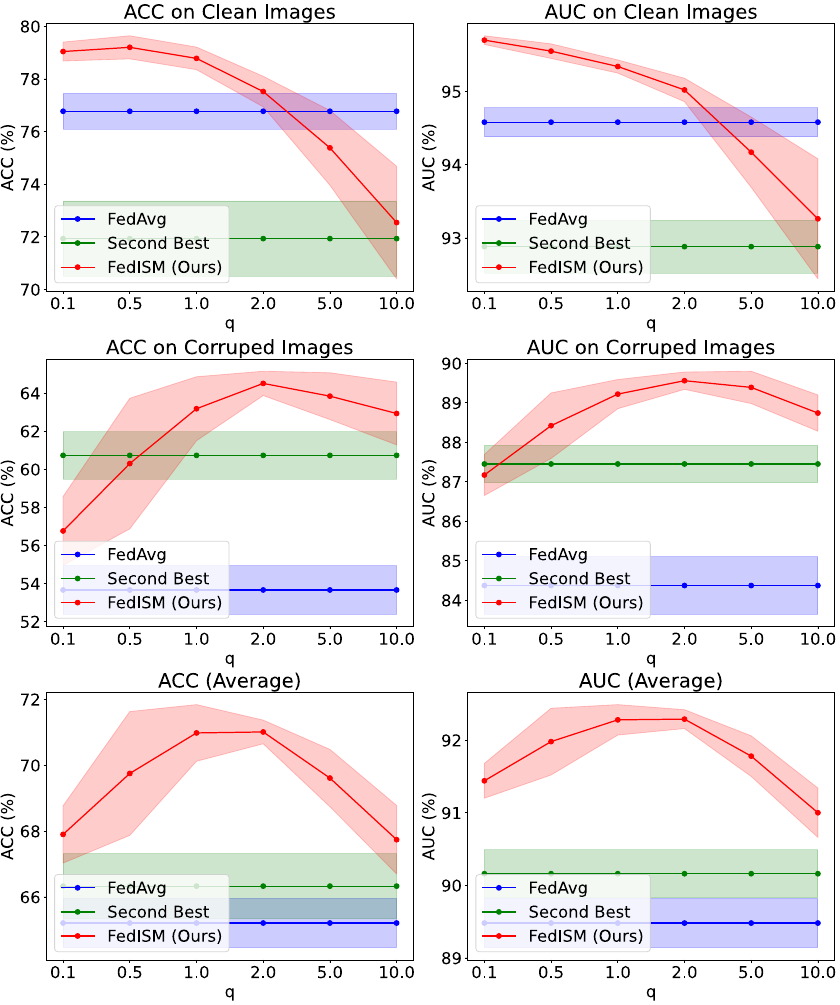}
		\caption{Evaluation on different $q$. Solid lines denote the mean values, while the transparent areas depict the standard deviations. Second best refers to the second best fair FL methods from Section \ref{sec:sota}, specifically in terms of their performance on corrupted images.}
		\label{fig:q}
	\end{figure}

	\section{Conclusion}
	In this paper, we pioneer the identification and formulation of a new fairness challenge in FL, specifically concerning imaging quality shifts across clients. To address this problem, existing FL approaches have primarily concentrated on balancing the empirical risk among distinct client distributions. Despite their effectiveness for fair optimization, they often neglect the crucial aspect of fair generalization. To address this overlooked area, we introduce Federated Learning with Inter-client Sharpness Matching (FedISM). FedISM innovatively refines both local training and global aggregation by integrating sharpness-awareness, effectively harmonizing sharpness levels across clients to achieve fair generalization in FL. Extensive empirical evaluations, on the well-recognized RSNA ICH and ISIC 2019 datasets, clearly demonstrate FedISM's superiority over current state-of-the-art FL methods in terms of promoting fairness. It highlights the effectiveness of FedISM in resolving fairness issues stemming from imaging quality shifts across medical datasets. We are confident that our proposed challenge and solution will pave the way for more equitable and effective FL system designs in medical applications and beyond.

	\clearpage
	
	\section*{Acknowledgments}
	This work was supported in part by the National Natural Science Foundation of China under Grants 62202179 and Grant 62271220, and in part by the Natural Science Foundation of Hubei Province of China under Grant 2022CFB585. The computation is supported by the HPC Platform of HUST.
	
	\bibliographystyle{named}
	\bibliography{ijcai24}

	\clearpage
	\appendix
	\part*{\centering Appendix}
	
	This appendix is organized as follows:
	\begin{itemize}
		
		\item In Section \ref{sec:appendix_notation}, we summarize the mathematical notations.
		
		\item In Section \ref{sec:appendix_theoretical_proof}, the theoretical proof regarding the theorem in this paper is given.
		
		\item In Section \ref{sec:appendix_experiments}, experimental details and additional results are presented.
	\end{itemize}

	\section{Notation Table} \label{sec:appendix_notation}
	
	\begin{table}[h]
		\centering
		\resizebox{1.0\columnwidth}{!}{
			\begin{tabular}{c|l}
				\toprule
				\hline
				Notations                         & Description                                                           \\ \hline
				$C$                               & Class Number                                                          \\
				$K$                               & Client Number                                                         \\
				$D_k$                             & Dataset of $k$-th Client                                              \\
				$\boldsymbol x$                   & Image                                                                 \\
				$y$                               & Label                                                                 \\
				$\mathcal{X}$                     & Image Space                                                           \\
				$\mathcal{Y}$                     & Label Space                                                           \\
				$N_k$                             & Size of the Local Dataset in Client $k$                               \\
				$a_k$                             & Imaging Quality of $k$-th Client                                      \\
				$A$                               & Number of Imaging Quality Categories                                  \\
				$f$                               & Neural Network                                                        \\
				$\ell$                            & Loss Function                                                         \\
				$\boldsymbol \theta$              & Weights of the Neural Network                                              \\
				$\Delta$                          & Probability Simplex                                                   \\
				$\mathbb{R_+}$                    & Set of Positive Real Numbers                                          \\
				$\boldsymbol \lambda$             & Linear Combination Weights of Clients                                 \\
				$\boldsymbol \mu$                 & Linear Combination Weights of Imaging Qualities                       \\
				$\mathds{1}$                     & Indicator Function                                                    \\
				$\mathcal{S}$                     & Sharpness                                                             \\
				$\boldsymbol \epsilon$            & Weights Perturbation                                                  \\
				$\rho$                            & Searching Distance                                                      \\
				$\eta$                            & Learning Rate of Local Training                                                         \\
				$T$                               & Total Communication Rounds                                            \\
				$\boldsymbol w_{\texttt{Avg}}$ & Aggregation Weights for FedAvg                                        \\
				$\widetilde{\boldsymbol w}_{t}$   & Aggregation Weights of FedISM in round $t$                     \\
				$\boldsymbol w_{t}$               & Aggregation Weights of FedISM in round $t$ with Moving Average \\
				$q$                               & Parameter for Global Aggregation                                      \\
				$\beta$                           & Parameter for Moving Average                                          \\ \hline
				\bottomrule
			\end{tabular}
		}
		\caption{Mathematical notations in this paper.}
	\end{table}

	\section{Theoretical Proof} \label{sec:appendix_theoretical_proof}
	
	\setcounter{theorem}{0}
	\setcounter{equation}{1}
	
	In this section, the proof of the theorem is given.
	
	\begin{theorem}[Equivalence] \label{eq:appdendix_Equivalence}
		Assuming class distributions of the testing set and all clients' training sets are identical, we have:
		\begin{equation} \label{eq:appendix_client_fairness}
			\resizebox{0.9\hsize}{!}{$
				\begin{aligned}
					&\boldsymbol \theta^*, \boldsymbol \lambda^* = \arg \mathop{\min}\limits_{\boldsymbol \theta} \mathop{\max}\limits_{\boldsymbol \lambda \in \Delta^{K-1}} \sum_{k=1}^{K} \lambda_k\mathbb{E}_{(\boldsymbol x,y) \sim \mathbb{P}_k(\boldsymbol x, y \mid a_k)} [\ell(f(\boldsymbol x; \boldsymbol \theta), y)], 
				\end{aligned}
				$}
		\end{equation}
		and
		\begin{equation} \label{eq:appendix_group_fairness}
			\resizebox{0.9\hsize}{!}{$
				\begin{aligned}
					&\boldsymbol \theta^*, \boldsymbol \mu^* = \arg \mathop{\min}\limits_{\boldsymbol \theta} \mathop{\max}\limits_{\boldsymbol \mu \in \Delta^{A-1}} \sum_{u=1}^{A} \mu_u \mathbb{E}_{(\boldsymbol x,y) \sim \mathbb{P}(\boldsymbol x, y \mid u)} [\ell(f(\boldsymbol x; \boldsymbol \theta), y)],
				\end{aligned}
				$}
		\end{equation}
		where $\mu^*_t = \sum_{k=1}^{K} \mathds{1}_{a_k=u} \lambda^*_k$.
	\end{theorem}
	
	\begin{proof}[\textbf{Proof}]\renewcommand{\qedsymbol}{}
		Eq. \ref{eq:appendix_client_fairness} can be re-written as:
		\begin{equation} \nonumber
			\begin{aligned}
				&\mathop{\max}\limits_{\boldsymbol \lambda \in \Delta^{K-1}} \sum_{k=1}^{K} \lambda_k\mathbb{E}_{(\boldsymbol x,y) \sim \mathbb{P}_k(\boldsymbol x, y \mid a_k)} [\ell(f(\boldsymbol x; \boldsymbol \theta), y)] \\
				=&\mathop{\max}\limits_{\boldsymbol \lambda \in \Delta^{K-1}} \sum_{k=1}^{K} \lambda_k \sum_{u=1}^{A} \mathds{1}_{a_k=u} \mathbb{E}_{(\boldsymbol x,y) \sim \mathbb{P}_k(\boldsymbol x, y \mid u)} [\ell(f(\boldsymbol x; \boldsymbol \theta), y)]] \\
				=&\mathop{\max}\limits_{\boldsymbol \lambda \in \Delta^{K-1}} \sum_{u=1}^{A} (\sum_{k=1}^{K} \mathds{1}_{a_k=u} \lambda_k) \mathbb{E}_{(\boldsymbol x,y) \sim \mathbb{P}_k(\boldsymbol x, y \mid u)} [\ell(f(\boldsymbol x; \boldsymbol \theta), y)]]
			\end{aligned}
		\end{equation}
		We define $\mu_u := \sum_{k=1}^{K} \mathds{1}_{a_k=u} \lambda_k$ and $\boldsymbol \mu=[\mu_1, \dots, \mu_A]^\top$. For each $u \in [A]$, please note that there is at lease one client $k$ satisfying $\mathds{1}_{a_k=u}=1$. Hence, we can obtain all the standard orthogonal bases of $\mathbb{R}^A$ space by changing the preimage $\boldsymbol \lambda \in \Delta^{K-1}$. This indicates that, for any $\boldsymbol \mu \in \Delta^{A-1}$, it can be acquired by the linear combination of these preimages. Therefore, we have:
		\begin{equation} \nonumber
			\begin{aligned}
				&\mathop{\max}\limits_{\boldsymbol \lambda \in \Delta^{K-1}} \sum_{u=1}^{A} (\sum_{k=1}^{K} \mathds{1}_{a_k=u} \lambda_k) \mathbb{E}_{(\boldsymbol x,y) \sim \mathbb{P}_k(\boldsymbol x, y \mid u)} [\ell(f(\boldsymbol x; \boldsymbol \theta), y)]] \\
				= &\mathop{\max}\limits_{\boldsymbol \mu \in \Delta^{A-1}} \sum_{u=1}^{A} \mu_u \mathbb{E}_{(\boldsymbol x,y) \sim \mathbb{P}(\boldsymbol x, y \mid u)} [\ell(f(\boldsymbol x; \boldsymbol \theta), y)]
			\end{aligned}
		\end{equation}
	\end{proof}
	This completes the proof of the theorem.

	\begin{table*}[!t]
		\centering
		\resizebox{1.0\textwidth}{!}{
			\begin{tabular}{c|cccccccccccc}
				\toprule
				\hline
				\multirow{3}{*}{Method} & \multicolumn{12}{c}{ICH (Corrupted Clients Ratio: 30\%)}                                                                                                                                                                                                                                                        \\ \cline{2-13} 
				& \multicolumn{4}{c|}{Clean}                                                                                 & \multicolumn{4}{c|}{Crorrupted}                                                                            & \multicolumn{4}{c}{Average}                                                           \\ \cline{2-13} 
				& ACC $\uparrow$ & $\Delta$ACC $\uparrow$   & AUC $\uparrow$ & \multicolumn{1}{c|}{$\Delta$AUC $\uparrow$}   & ACC $\uparrow$ & $\Delta$ACC $\uparrow$   & AUC $\uparrow$ & \multicolumn{1}{c|}{$\Delta$AUC $\uparrow$}   & ACC $\uparrow$ & $\Delta$ACC $\uparrow$   & AUC $\uparrow$ & $\Delta$AUC $\uparrow$   \\ \hline
				FedAvg                  & $75.56$        & \multirow{2}{*}{-}       & $94.36$        & \multicolumn{1}{c|}{\multirow{2}{*}{-}}       & $57.14$        & \multirow{2}{*}{-}       & $85.95$        & \multicolumn{1}{c|}{\multirow{2}{*}{-}}       & $66.35$        & \multirow{2}{*}{-}       & $90.16$        & \multirow{2}{*}{-}       \\
				(AISTATS'17)            & $(0.69)$       &                          & $(0.22)$       & \multicolumn{1}{c|}{}                         & $(1.86)$       &                          & $(0.83)$       & \multicolumn{1}{c|}{}                         & $(0.88)$       &                          & $(0.41)$       &                          \\ \hline
				FedAvg                  & $78.57$        & \multirow{2}{*}{$+3.01$} & $\bm{95.63}$   & \multicolumn{1}{c|}{\multirow{2}{*}{$\bm{+1.27}$}} & $62.16$        & \multirow{2}{*}{$+5.02$} & $89.14$        & \multicolumn{1}{c|}{\multirow{2}{*}{$+3.19$}} & $70.37$        & \multirow{2}{*}{$+4.02$} & $92.39$        & \multirow{2}{*}{$+2.23$} \\
				+SALT                   & $(0.50)$       &                          & $\bm{(0.12)}$  & \multicolumn{1}{c|}{}                         & $(0.80)$       &                          & $(0.20)$       & \multicolumn{1}{c|}{}                         & $(0.53)$       &                          & $(0.12)$       &                          \\ \hline
				FedAvg                  & $75.70$        & \multirow{2}{*}{$+0.14$} & $94.08$        & \multicolumn{1}{c|}{\multirow{2}{*}{$-0.28$}} & $62.41$        & \multirow{2}{*}{$+5.27$} & $88.20$        & \multicolumn{1}{c|}{\multirow{2}{*}{$+2.25$}} & $69.06$        & \multirow{2}{*}{$+2.71$} & $91.14$        & \multirow{2}{*}{$+0.98$} \\
				+SAGA                   & $(0.70)$       &                          & $(0.26)$       & \multicolumn{1}{c|}{}                         & $(1.89)$       &                          & $(0.67)$       & \multicolumn{1}{c|}{}                         & $(1.01)$       &                          & $(0.32)$       &                          \\ \hline
				FedISM                  & $\bm{78.74}$        & \multirow{2}{*}{$\bm{+3.18}$}  & $95.32$        & \multicolumn{1}{c|}{\multirow{2}{*}{$+0.96$}} & $\bm{66.86}$   & \multirow{2}{*}{$\bm{+9.72}$} & $\bm{90.66}$   & \multicolumn{1}{c|}{\multirow{2}{*}{$\bm{+4.71}$}} & $\bm{72.80}$   & \multirow{2}{*}{$\bm{+6.45}$} & $\bm{92.99}$   & \multirow{2}{*}{$\bm{+2.83}$} \\
				(Ours)                  & $\bm{(0.56)}$       &                          & $(0.23)$       & \multicolumn{1}{c|}{}                         & $\bm{(1.22)}$  &                          & $\bm{(0.18)}$  & \multicolumn{1}{c|}{}                         & $\bm{(0.44)}$  &                          & $\bm{(0.10)}$  &                          \\ \hline
				\bottomrule
			\end{tabular}
		}
		\caption{Component-wise ablation study in mean (\%) and standard deviation (\%) in ACC and AUC. Values without parentheses represent the mean, while those within parentheses indicate the standard deviation. $\Delta$ACC and $\Delta$AUC represent the difference to ACC and AUC of FedAvg, respectively. The ratio of corrupted clients is 30\%.}
		\label{tab:appendix:ablation1}
	\end{table*}
	
	\begin{table*}[!t]
		\centering
		\resizebox{1.0\textwidth}{!}{
			\begin{tabular}{c|cccccccccccc}
				\toprule
				\hline
				\multirow{3}{*}{Method} & \multicolumn{12}{c}{ICH (Corrupted Clients Ratio: 10\%)}                                                                                                                                                                                                                                                                                       \\ \cline{2-13} 
				& \multicolumn{4}{c|}{Clean}                                                                                           & \multicolumn{4}{c|}{Crorrupted}                                                                                       & \multicolumn{4}{c}{Average}                                                                     \\ \cline{2-13} 
				& ACC $\uparrow$ & $\Delta$ACC $\uparrow$        & AUC $\uparrow$ & \multicolumn{1}{c|}{$\Delta$AUC $\uparrow$}        & ACC $\uparrow$ & $\Delta$ACC $\uparrow$         & AUC $\uparrow$ & \multicolumn{1}{c|}{$\Delta$AUC $\uparrow$}        & ACC $\uparrow$ & $\Delta$ACC $\uparrow$        & AUC $\uparrow$ & $\Delta$AUC $\uparrow$        \\ \hline
				FedAvg                  & $78.17$        & \multirow{2}{*}{-}            & $95.26$        & \multicolumn{1}{c|}{\multirow{2}{*}{-}}            & $42.72$        & \multirow{2}{*}{-}             & $78.26$        & \multicolumn{1}{c|}{\multirow{2}{*}{-}}            & $60.44$        & \multirow{2}{*}{-}            & $86.76$        & \multirow{2}{*}{-}            \\
				(AISTATS'17)            & $(0.55)$       &                               & $(0.07)$       & \multicolumn{1}{c|}{}                              & $(2.12)$       &                                & $(1.35)$       & \multicolumn{1}{c|}{}                              & $(1.05)$       &                               & $(0.68)$       &                               \\ \hline
				FedAvg                  & $\bm{80.04}$   & \multirow{2}{*}{$\bm{+1.90}$} & $\bm{96.00}$   & \multicolumn{1}{c|}{\multirow{2}{*}{$\bm{+0.74}$}} & $46.91$        & \multirow{2}{*}{$+4.19$}       & $83.37$        & \multicolumn{1}{c|}{\multirow{2}{*}{$+5.11$}}      & $63.48$        & \multirow{2}{*}{$+3.04$}      & $89.68$        & \multirow{2}{*}{$+2.92$}      \\
				+SALT                   & $\bm{(0.34)}$  &                               & $\bm{(0.11)}$  & \multicolumn{1}{c|}{}                              & $(1.28)$       &                                & $(0.82)$       & \multicolumn{1}{c|}{}                              & $(0.67)$       &                               & $(0.43)$       &                               \\ \hline
				FedAvg                  & $75.35$        & \multirow{2}{*}{$-2.79$}      & $94.16$        & \multicolumn{1}{c|}{\multirow{2}{*}{$-1.10$}}      & $55.15$        & \multirow{2}{*}{$+12.43$}      & $84.68$        & \multicolumn{1}{c|}{\multirow{2}{*}{$+6.42$}}      & $65.25$        & \multirow{2}{*}{$+4.81$}      & $89.42$        & \multirow{2}{*}{$+2.66$}      \\
				+SAGA                   & $(1.03)$       &                               & $(0.28)$       & \multicolumn{1}{c|}{}                              & $(2.37)$       &                                & $(1.06)$       & \multicolumn{1}{c|}{}                              & $(1.26)$       &                               & $(0.55)$       &                               \\ \hline
				FedISM                  & $78.91$        & \multirow{2}{*}{$+0.77$}      & $95.54$        & \multicolumn{1}{c|}{\multirow{2}{*}{$+0.28$}}      & $\bm{58.84}$   & \multirow{2}{*}{$\bm{+16.12}$} & $\bm{87.61}$   & \multicolumn{1}{c|}{\multirow{2}{*}{$\bm{+9.35}$}} & $\bm{68.88}$   & \multirow{2}{*}{$\bm{+8.44}$} & $\bm{91.57}$   & \multirow{2}{*}{$\bm{+4.81}$} \\
				(Ours)                  & $(0.59)$       &                               & $(0.10)$       & \multicolumn{1}{c|}{}                              & $\bm{(1.72)}$  &                                & $\bm{(0.46)}$  & \multicolumn{1}{c|}{}                              & $\bm{(1.03)}$  &                               & $\bm{(0.25)}$  &                               \\ \hline
				\bottomrule
			\end{tabular}
		}
		\caption{Component-wise ablation study in mean (\%) and standard deviation (\%) in ACC and AUC. Values without parentheses represent the mean, while those within parentheses indicate the standard deviation. $\Delta$ACC and $\Delta$AUC represent the difference to ACC and AUC of FedAvg, respectively. The ratio of corrupted clients is 10\%.}
		\label{tab:appendix:ablation2}
	\end{table*}

	\section{Experiments} \label{sec:appendix_experiments}
	\subsection{Details of SOTA Methods for Comparison}
	This section provides an overview and implementation approaches of the methods compared in Section 4.2.
	
	\begin{itemize}
		\item Agnostic-FL \cite{mohri2019agnostic}: Agnostic-FL, a pioneering work in fair FL, introduces a training strategy that focuses on updating only the poorest-performing client. This approach aims to enhance fairness, albeit with potential drawbacks such as slower convergence and model instability.
		
		\item q-FedAvg \cite{li2019fair}: q-FedAvg addresses fairness issues by incorporating training loss into the global update process. To improve model convergence and maintain consistency with other methods, we have adjusted the initial global update method by introducing a multiplicative constant. This adjustment aligns it with FedAvg, employing loss-aware aggregation weights. For parameter optimization, we experimented with different values of the parameter $q$ in q-FedAvg, specifically \{0.5, 1.0, 2.0, 5.0\}, and report the configuration yielding the best performance.
		
		\item FairFed \cite{ezzeldin2023fairfed}: FairFed aims at achieving group fairness through equitable optimization across clients. The parameter $\beta$ in FairFed is tuned from \{0.1, 0.5, 1.0\}.
		
		\item FedCE \cite{jiang2023fair}: FedCE focuses on promoting fair FL specifically in the context of medical image segmentation. For the purposes of this study, we have adapted it for image classification tasks.
		
		\item FedGA \cite{zhang2023federated}: This method is designed for better domain generalization by pursuing fairness across different clients/domains.
	\end{itemize}

	\subsection{Additional Experimental Results}
	\subsubsection{Ablation Study}
	Component-wise ablation studies are conducted in two additional settings with varying ratios of clean and corrupted clients. Specifically, the ratio of corrupted clients is altered to 30\% and 10\%, while keeping the corruption type the same as in Tab. \ref{tab:Ablation}. The results are summarized in Tab. \ref{tab:appendix:ablation1} and Tab. \ref{tab:appendix:ablation2}. Both experiments consistently indicate the same conclusion as the main text. SALT enhances generalization but does not significantly improve fairness. SAGA solely balances importance weights without directly addressing sharpness minimization during training. The optimal performance is achieved by combining these two components.
	
	\subsubsection{Discussion on Parameter $\beta$}
	To ensure training stability, we employ a moving average approach in our FedISM as follows:
	\begin{equation} \nonumber
		{\boldsymbol w}_{t} = \beta \widetilde{\boldsymbol w}_{t} + (1-\beta) {\boldsymbol w}_{t-1}.
	\end{equation}
	In this study, we explore the impact of the parameter $\beta$. We vary $\beta$ through the values \{0.3, 0.5, 0.7, 0.9, 1.0\} and conduct experiments on the ICH dataset, using the approach described in Sec. \ref{sec:sota}. The results of these experiments are depicted in Fig. \ref{fig:appendix_beta}. Additionally, we include the performance metrics of FedAvg and the second-best method for enhanced comparison. Notably, compared to the scenario where moving average is not applied (\textit{i.e.}, $\beta=1.0$), the incorporation of moving average yields improved and more consistent performance, typically characterized by a reduced standard deviation. Moreover, we observe that the performance is relatively insensitive to different values of $\beta$ when moving average is integrated.
	
	\begin{figure*}[!t] 
		\centering
		\includegraphics[width=1.0\textwidth]{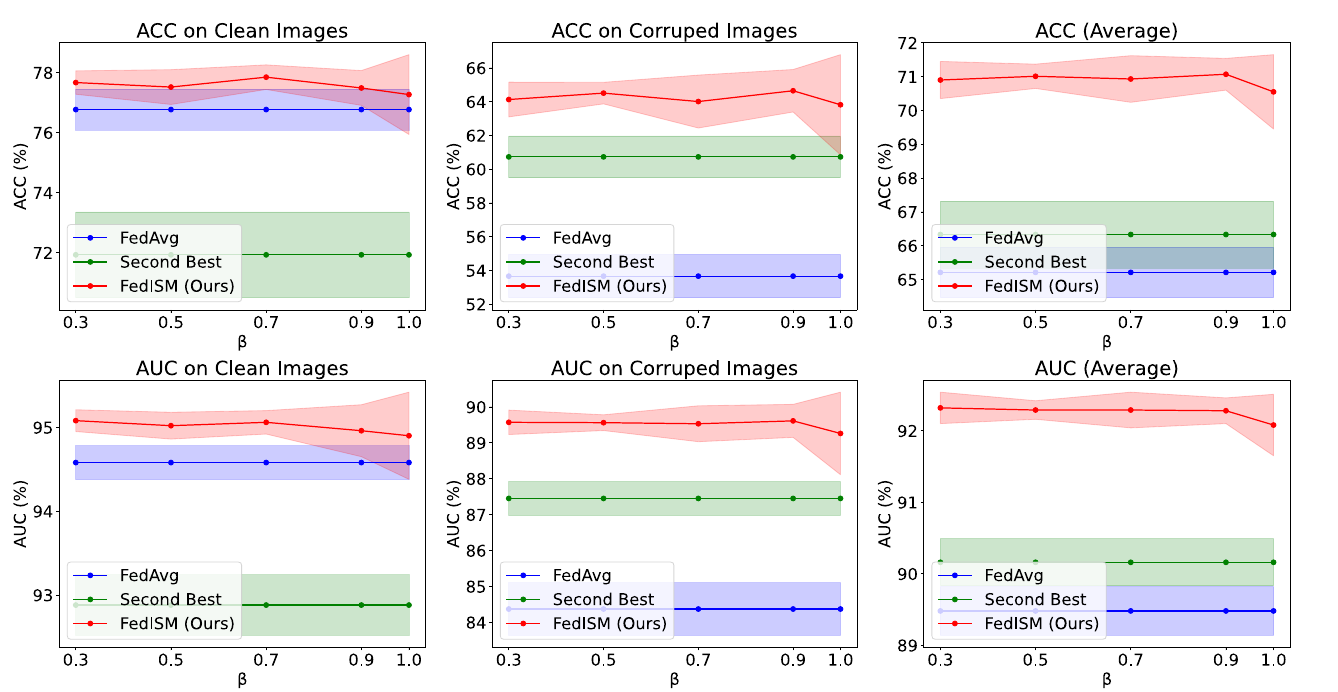}
		\caption{Evaluation on different $\beta$. Solid lines denote the mean values, while the transparent areas depict the standard deviations. Second best refers to the second best fair FL methods from Section \ref{sec:sota}, specifically in terms of their performance on corrupted images.}
		\label{fig:appendix_beta}
	\end{figure*}

	\subsubsection{Sharpness across Clients}
	We present the statistical information (\textit{i.e.}, mean and stand deviation) of clients' sharpness as shown in Tab. \ref{tab:appendix_sharpness_distribution}. Our proposed method, FedISM, demonstrates a more uniform and lower sharpness across clients compared to other methods. This aligns well with our conceptual framework of inter-client sharpness matching (\textit{i.e.}, ensuring consistent and low sharpness levels across different clients).

	\begin{table*}[!t]
		\centering
		\begin{tabular}{l|ccccccc}
			\toprule
			\hline
			Method & \multicolumn{1}{l}{FedAvg} & \multicolumn{1}{l}{Agnostic-FL} & \multicolumn{1}{l}{q-FedAvg} & \multicolumn{1}{l}{FairFed} & \multicolumn{1}{l}{FedCE} & \multicolumn{1}{l}{FedGA} & \multicolumn{1}{l}{FedISM (Ours)} \\ \hline
			Mean $\downarrow$  & $0.67370$                  & $1.01556$                       & $0.63291$                    & $1.07346$                   & $0.66945$                 & $1.06669$                 & $\bm{0.44165}$                         \\
			Std $\downarrow$    & $0.56693$                  & $0.65518$                       & $0.16827$                    & $0.22205$                   & $0.29507$                 & $0.28508$                 & $\bm{0.09966}$                         \\ \hline
			\bottomrule
		\end{tabular}
		\caption{Mean and standard deviation of sharpness across clients.}
		\label{tab:appendix_sharpness_distribution}
	\end{table*}

	\subsubsection{Visualization of Loss landscape}
	The loss landscape under model weight perturbation is visualized as Fig. \ref{fig:appendix_loss_landscape} following [Li \textit{et al.}, 2018], which effectively captures the dynamics of loss variation in the vicinity of the model's convergence point. Traditional approaches for fair optimization primarily aim to reduce training loss across various clients, often neglecting the geometric characteristics of the loss surface. This can result in convergence at sharper minima, as illustrated by the two training columns in the figure. However, our proposed method, FedISM, diverges from this norm. By targeting uniformly low sharpness across both clean and corrupted image clients, FedISM consistently achieves convergence at flatter minima. This strategy fosters superior generalization, as evidenced by the comparatively lower testing losses (see two columns of testing).
	
	\begin{figure*}[!t] 
		\centering
		\includegraphics[width=1.0\textwidth]{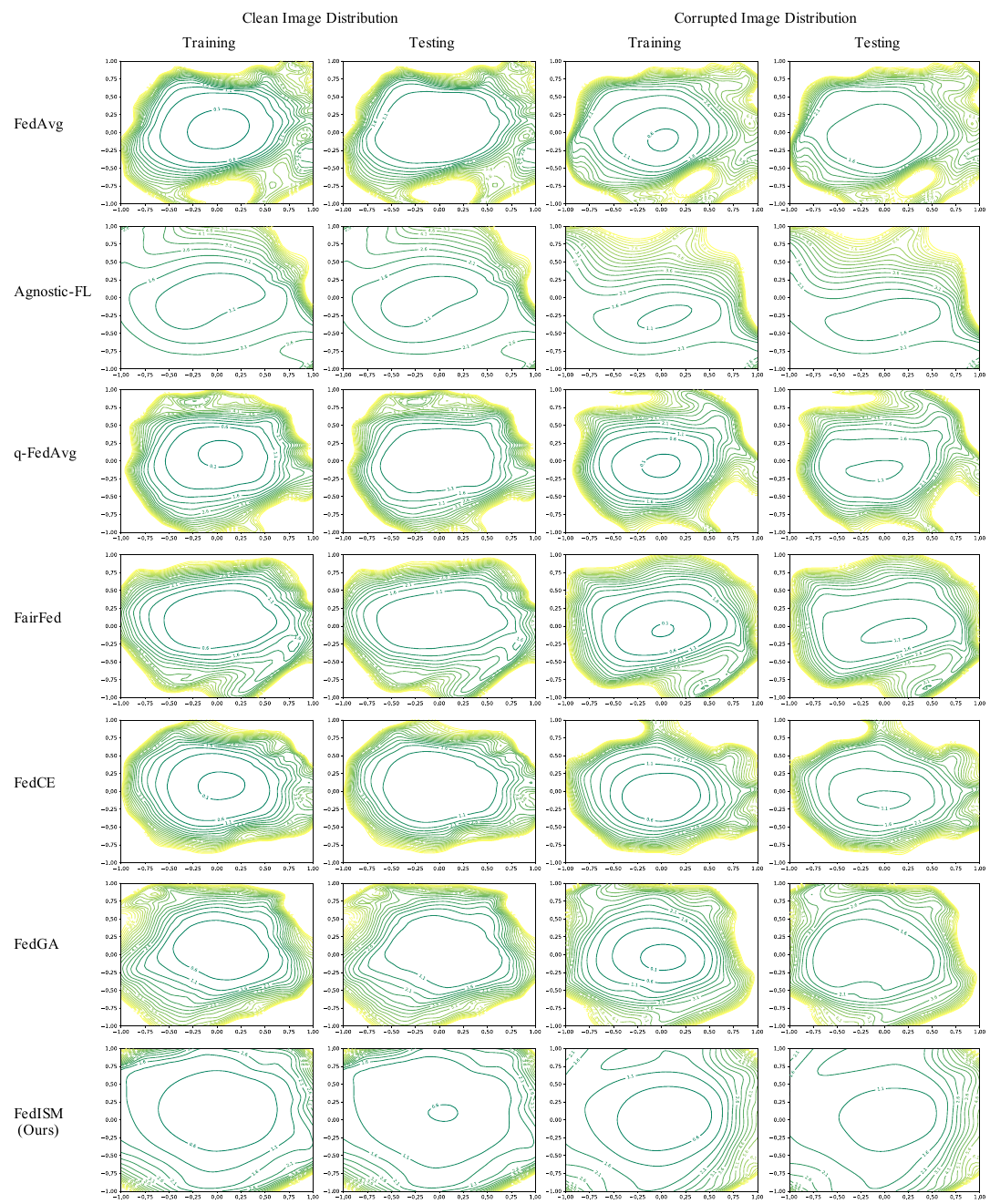}
		\caption{Loss landscape of different methods on training and testing sets of two image distributions. The origin (\textit{i.e.}, $(x=0.0, y=0.0)$) represents the initial model, while other points mean models that have been perturbed along two axes (i.e., $x$ and $y$).}
		\label{fig:appendix_loss_landscape}
	\end{figure*}

	\subsubsection{Evaluation with Different Corruption Type}
	Different experiments in Sec. \ref{sec:sota} that model corruption as Gaussian noise, we additionally generate corrupted images by random motion blur \cite{hendrycks2019benchmarking}. This is also a typical corruption in imaging caused by random movements from either patients or cameras. Following Sec. \ref{sec:sota}, the ratio of corrupted clients is set as 20\%. We summarize the results as Tab. \ref{tab:appendix_SOTA_motionBlur}, which can see the consistent superiority of our method. Here, the parameter of GSAM \cite{zhuang2022surrogate} is tuned to enhance absolute classification performance, indicating the potential impact of sharpness calculation in our method. This may inspire the next steps in our research on sharpness definition, which may further enhance our method.
	
	\begin{table*}[!t]
		\centering
		\resizebox{1.0\textwidth}{!}{
			\begin{tabular}{c|c|cccccccccccc}
				\toprule
				\hline
				\multirow{4}{*}{Category} & \multirow{4}{*}{Method} & \multicolumn{12}{c}{Dataset (Corrupted Clients Ratio: 20\%)} \\ \cline{3-14} 
				& & \multicolumn{6}{c|}{ICH} & \multicolumn{6}{c}{ISIC 2019} \\ \cline{3-14} 
				& & \multicolumn{2}{c}{Clean} & \multicolumn{2}{c}{Corrupted} & \multicolumn{2}{c|}{Average} & \multicolumn{2}{c}{Clean} & \multicolumn{2}{c}{Corrupted} & \multicolumn{2}{c}{Average} \\ \cline{3-14} 
				& & ACC $\uparrow$ & AUC $\uparrow$ & ACC $\uparrow$ & AUC $\uparrow$ & ACC $\uparrow$ & \multicolumn{1}{c|}{AUC $\uparrow$} & ACC $\uparrow$ & AUC $\uparrow$ & ACC $\uparrow$ & AUC $\uparrow$ & ACC $\uparrow$ & AUC $\uparrow$ \\ \hline
				\multirow{2}{*}{Naive FL} & FedAvg & $75.97$ & $94.74$ & $68.69$ & $90.98$ & $72.33$ & \multicolumn{1}{c|}{$92.86$} & $67.16$ & $92.60$ & $54.98$ & $88.16$ & $61.07$ & $90.38$ \\
				& (AISTATS'17) & $(1.11)$ & $(0.22)$ & $(0.74)$ & $(0.41)$ & $(0.81)$ & \multicolumn{1}{c|}{$(0.29)$} & $(1.15)$ & $(0.37)$ & $(1.72)$ & $(0.54)$ & $(1.31)$ & $(0.41)$ \\ \hline
				\multirow{10}{*}{Fair FL} & Agnostic-FL & $63.57$ & $89.45$ & $57.27$ & $86.35$ & $60.42$ & \multicolumn{1}{c|}{$87.90$} & $51.24$ & $87.90$ & $50.11$ & $86.56$ & $50.67$ & $87.23$ \\
				& (ICML'19) & $(3.48)$ & $(1.32)$ & $(7.12)$ & $(3.32)$ & $(3.27)$ & \multicolumn{1}{c|}{$(1.78)$} & $(5.49)$ & $(1.91)$ & $(3.85)$ & $(1.64)$ & $(2.93)$ & $(1.18)$ \\ \cline{2-14} 
				& q-FedAvg & $75.68$ & $94.77$ & $70.53$ & $91.87$ & $73.10$ & \multicolumn{1}{c|}{$93.32$} & $71.39$ & $93.88$ & $60.79$ & $90.03$ & $66.09$ & $91.95$ \\
				& (ICLR'20) & $(1.59)$ & $(0.34)$ & $(1.18)$ & $(0.53)$ & $(0.54)$ & \multicolumn{1}{c|}{$(0.17)$} & $(0.73)$ & $(0.30)$ & $(1.38)$ & $(0.39)$ & $(0.96)$ & $(0.30)$ \\ \cline{2-14} 
				& FairFed & $74.59$ & $94.34$ & $69.66$ & $91.75$ & $72.12$ & \multicolumn{1}{c|}{$93.05$} & $63.94$ & $91.89$ & $56.27$ & $89.58$ & $60.11$ & $90.73$ \\
				& (AAAI'23) & $(1.12)$ & $(0.21)$ & $(1.08)$ & $(0.30)$ & $(0.66)$ & \multicolumn{1}{c|}{$(0.18)$} & $(1.58)$ & $(0.48)$ & $(1.59)$ & $(0.58)$ & $(1.48)$ & $(0.49)$ \\ \cline{2-14}
				& FedCE & $76.83$ & $95.00$ & $70.38$ & $92.04$ & $73.60$ & \multicolumn{1}{c|}{$93.52$} & $72.29$ & $94.26$ & $61.85$ & $90.95$ & $67.07$ & $92.61$ \\
				& (CVPR'23) & $(0.83)$ & $(0.15)$ & $(0.72)$ & $(0.28)$ & $(0.73)$ & \multicolumn{1}{c|}{$(0.17)$} & $(0.64)$ & $(0.28)$ & $(1.28)$ & $(0.40)$ & $(0.74)$ & $(0.31)$ \\ \cline{2-14}
				& FedGA & $72.97$ & $93.75$ & $69.64$ & $91.78$ & $71.31$ & \multicolumn{1}{c|}{$92.76$} & $66.35$ & $92.78$ & $59.22$ & $90.20$ & $62.79$ & $91.49$ \\
				& (CVPR'23) & $(0.78)$ & $(0.16)$ & $(0.60)$ & $(0.27)$ & $(0.54)$ & \multicolumn{1}{c|}{$(0.16)$} & $(1.53)$ & $(0.45)$ & $(1.46)$ & $(0.36)$ & $(1.31)$ & $(0.34)$ \\ \hline
				\multirow{2}{*}{Ours} & \multirow{2}{*}{FedISM} & $\bm{77.86}$ & $\bm{95.60}$ & $\bm{72.38}$ & $\bm{93.59}$ & $\bm{75.12}$ & \multicolumn{1}{c|}{$\bm{94.60}$} & $\bm{73.44}$ & $\bm{95.07}$ & $\bm{63.45}$ & $\bm{92.21}$ & $\bm{68.45}$ & $\bm{93.64}$ \\
				& & $\bm{(0.67)}$ & $\bm{(0.07)}$ & $\bm{(0.75)}$ & $\bm{(0.18)}$ & $\bm{(0.53)}$ & \multicolumn{1}{c|}{$\bm{(0.08)}$} & $\bm{(0.88)}$ & $\bm{(0.18)}$ & $\bm{(0.75)}$ & $\bm{(0.29)}$ & $\bm{(0.68)}$ & $\bm{(0.16)}$ \\ \hline
				\bottomrule
			\end{tabular}
		}
		\caption{Comparison of the mean (\%) and standard deviation (\%) for ACC and AUC against state-of-the-art methods. Values without parentheses represent the mean, while those within parentheses indicate the standard deviation. The type of corruption is motion blur.}
		\label{tab:appendix_SOTA_motionBlur}
	\end{table*}
	
	\section*{Additional References for Appendix}
	
	[Li \textit{et al.}, 2018] Hao Li, Zheng Xu, Gavin Taylor, Christoph Studer and Tom Goldstein. Visualizing the Loss Landscape of Neural Nets. In \textit{NeurIPS}, 2018.

\end{document}